\documentclass{llncs}

\usepackage{amsfonts}
\usepackage{amsmath}
\usepackage{latexsym}
\usepackage{color}
\usepackage{graphicx}
\usepackage{enumerate}
\usepackage{amstext}
\usepackage{epsfig}
\usepackage{natbib}
\usepackage[whileod]{algochl}

\def\Rset{\mathbb{R}}

\DeclareMathOperator*{\E}{\rm E}

\DeclareMathOperator*{\argmin}{\rm argmin}
\DeclareMathOperator{\sgn}{sgn}

\DeclareMathOperator{\Tr}{Tr}

\newcommand{\mat}[1]{{\mathbf #1}}
\newcommand{\K}{\mat{K}}

\renewcommand{\L}{\mat{L}}

\newcommand{\h}{\widehat}

\renewcommand{\u}{\mat{u}}
\renewcommand{\v}{\mat{v}}

\newcommand{\w}{\mat{w}}
\newcommand{\x}{\mat{x}}

\newcommand{\z}{\mat{z}}

\newcommand{\1}{\mat{1}}
\newcommand{\0}{\mat{0}}

\newcommand{\Alpha}{{\boldsymbol \alpha}}

\newcommand{\Alg}{{\cal A}}

\newcommand{\ignore}[1]{}

\newcommand{\EQ}{\gets}
\newcommand{\TO}{\mbox{ {\bf to }}}

\title{Perceptron Mistake Bounds}

\author{Mehryar Mohri\inst{1,2} \and Afshin Rostamizadeh\inst{1}}
\institute{
  Google Research \and Courant Institute of Mathematical Sciences
}

\begin{document}
\maketitle

\abstract{We present a brief survey of existing mistake bounds and
  introduce novel bounds for the Perceptron or the kernel Perceptron
  algorithm.  Our novel bounds generalize beyond standard margin-loss
  type bounds, allow for any convex and Lipschitz loss function, and
  admit a very simple proof.}

\section{Introduction}

\begin{figure}[t]
\begin{ALGO}{Perceptron\index{Perceptron algorithm}}{\w_0}
\SET{\w_1}{\w_0 \qquad \triangleright \text{typically } \w_ 0 = \mat{0}}
\DOFOR{t \EQ 1 \TO T}
\CALL{Receive}{\x_t}
\SET{\h y_t}{\sgn(\w_t \cdot \x_t)}
\CALL{Receive}{y_t}
\IF{(\h y_t \neq y_t)}
	\SET{\w_{t + 1}}{\w_t + y_t \x_t \quad \triangleright  \text{more generally } \eta y_t \x_t, \eta > 0.}
\ELSE
	\SET{\w_{t + 1}}{\w_t}	
\FI
\OD
\RETURN{\w_{T + 1}}
\end{ALGO}
 \caption{Perceptron\index{Perceptron algorithm} algorithm
\citep{Rosenblatt58}.}
\label{fig:perceptron_algorithm}
\end{figure}

The Perceptron algorithm belongs to the broad family of on-line
learning algorithms (see \cite{CesaBianchiLugosi2006} for a survey)
and admits a large number of variants. The algorithm learns a linear
separator by processing the training sample in an on-line fashion,
examining a single example at each iteration \citep{Rosenblatt58}.  At
each round, the current hypothesis is updated if it makes a mistake,
that is if it incorrectly classifies the new training point
processed. The full pseudocode of the algorithm is provided in
Figure~\ref{fig:perceptron_algorithm}.  In what follows, we will
assume that $\w_0 = \0$ and $\eta = 1$ for simplicity of presentation,
however, the more general case also allows for similar guarantees
which can be derived following the same methods we are presenting.

This paper briefly surveys some existing \emph{mistake bounds} for the
Perceptron algorithm and introduces new ones which can be used to
derive generalization bounds in a stochastic setting.  A mistake bound
is an upper bound on the number of updates, or the number of mistakes,
made by the Perceptron algorithm when processing a sequence of
training examples. Here, the bound will be expressed in terms of the
performance of any linear separator, including the best. Such mistake
bounds can be directly used to derive generalization guarantees for
a combined hypothesis, using existing on-line-to-batch techniques.

\section{Separable case}

The seminal work of \citet{Novikoff62} gave the first margin-based
bound for the Perceptron algorithm, one of the early results in
learning theory and probably one of the first based on the notion of
margin.  Assuming that the data is separable with some margin $\rho$,
Novikoff showed that the number of mistakes made by the Perceptron
algorithm can be bounded as a function of the normalized margin
$\rho/R$, where $R$ is the radius of the sphere containing the
training instances.  We start with a Lemma that can be used to prove
Novikoff's theorem and that will be used throughout.
\begin{lemma}
\label{lemma:1}
Let $I$ denote the set of rounds at which the Perceptron algorithm makes
an update when processing a sequence of training instances $\x_1,
\ldots, \x_T \in \Rset^N$.  Then, the following inequality holds:
\begin{equation*}
  \Big\| \sum_{t \in I} y_t \x_t \Big\| \leq \sqrt{\sum_{t \in I} \|
\x_t\|^2} \,.
\end{equation*}
\end{lemma}
\begin{proof}
The inequality holds using the following sequence of observations,
\begin{align*}
\Big \| \sum_{t \in I} y_t \x_t \Big \|  
& = \Big \| \sum_{t \in I} (\w_{t + 1} - \w_t) \Big \| & (\text{definition of updates})\\
& = \| \w_{T + 1} \| & (\text{telescoping sum, $\w_0 = 0$})\\
& = \sqrt{\sum_{t \in I} \| \w_{t + 1} \|^2 - \| \w_t \|^2} & (\text{telescoping sum, $\w_0 = 0$})\\
& = \sqrt{\sum_{t \in I} \| \w_t  + y_t \x_t \|^2 - \| \w_t \|^2}  & (\text{definition of updates})\\
& = \sqrt{\sum_{t \in I} 2 \underbrace{y_t \w_t  \cdot \x_t}_{\leq 0}
  + \| \x_t \|^2} \\
& \leq \sqrt{\sum_{t \in I} \| \x_t \|^2} \,.
\end{align*}
The final inequality uses the fact that an update is made at round $t$ only
when the current hypothesis makes a mistake, that is, $y_t (\w_t \cdot
\x_t) \leq 0$. \qed
\end{proof}

The lemma can be used straightforwardly to derive the following mistake
bound for the separable setting.

\begin{theorem}[\citep{Novikoff62}]
\label{th:perceptron}
Let $\x_1, \ldots, \x_T \in \Rset^N$ be a sequence of $T$ points with
$\| \x_t \| \leq r$ for all $t \in [1, T]$, for some $r > 0$. Assume
that there exist $\rho > 0$ and $\v \in \Rset^N$, $\v \neq 0$, such
that for all $t \in [1, T]$, $\rho \leq \frac{y_t (\v \cdot \x_t)}{\|
  \v \|}$.  Then, the number of updates made by the
Perceptron\index{Perceptron algorithm} algorithm when processing
$\x_1, \ldots, \x_T$ is bounded by $r^2/\rho^2$.
\end{theorem}
\begin{proof}
Let $I$ denote the subset of the $T$ rounds at which there is an update,
and let $M$ be the total number of updates, i.e., $|I| = M$.
Summing up the inequalities yields:
\begin{align*}
M \rho 
\leq \frac{\v \cdot \sum_{t \in I} y_t \x_t}{\| \v \|}
 \leq \Big \| \sum_{t \in I} y_t \x_t \Big \| 
 \leq \sqrt{\sum_{t \in I} \| \x_t \|^2}
\leq \sqrt{M r^2},
\end{align*}
where the second inequality holds by the Cauchy-Schwarz inequality,
the third by Lemma~\ref{lemma:1} and the final one by assumption.
Comparing the left- and right-hand sides gives $\sqrt{M} \leq r/\rho$,
that is, $M \leq r^2/\rho^2$. \qed
\end{proof}

\section{Non-separable case}

In real-world problems, the training sample processed by the
Perceptron algorithm is typically not linearly
separable. Nevertheless, it is possible to give a margin-based mistake
bound in that general case in terms of the radius of the sphere
containing the sample and the margin-based loss of an arbitrary weight
vector. We present two different types of bounds: first, a bound that
depends on the $L_1$-norm of the vector of $\rho$-margin hinge losses,
or the vector of more general losses that we will describe, next a
bound that depends on the $L_2$-norm of the vector of margin losses,
which extends the original results presented by
\citet{FreundSchapire99}.

\subsection{$L_1$-norm mistake bounds}
\label{sec:l1}

We first present a simple proof of a mistake bound for the Perceptron
algorithm that depends on the $L_1$-norm of the losses incurred by an
arbitrary weight vector, for a general definition of the loss function
that covers the $\rho$-margin hinge loss. The family of
\emph{admissible} loss functions is quite general and defined as
follows.

\begin{definition}[$\gamma$-admissible loss function]
A \emph{$\gamma$-admissible} loss function $\phi_\gamma \colon
\Rset \to \Rset_+$ satisfies the following conditions:
\begin{enumerate}
\item The function $\phi_\gamma$ is convex.
\item $\phi_\gamma$ is non-negative: $\forall x
\in \Rset, \phi_\gamma(x) \geq 0$.
\item At zero, the $\phi_\gamma$ is strictly positive:
  $\phi_\gamma(0) > 0$.
\item $\phi_\gamma$ is $\gamma$-Lipschitz:
  $|\phi_\gamma(x) - \phi_\gamma(y)| \leq
\gamma |x - y|$, for some $\gamma > 0$.
\end{enumerate}
\end{definition}
These are mild conditions satisfied by many loss functions including
the hinge-loss, the squared hinge-loss, the Huber loss and general
$p$-norm losses over bounded domains.

\begin{theorem}
\label{th:l1}
Let $I$ denote the set of rounds at which the Perceptron algorithm
makes an update when processing a sequence of training instances
$\x_1, \ldots, \x_T \in \Rset^N$. For any vector $\u \in \Rset^{N}$ with $\|
\u \| \leq 1$ and any $\gamma$-admissible loss function $\phi_\gamma$,
consider the vector of losses incurred by $\u$:
$ \L_{\phi_\gamma}(\u) = \big[ \phi_\gamma(y_t (\u \cdot \x_t)) \big]_{t
\in I}$.
Then, the number of updates $M_T = |I|$ made by the Perceptron
algorithm can be bounded as follows:
\begin{equation}
M_T \leq \inf_{\gamma > 0, \|\u\| \leq 1}
\frac{1}{\phi_\gamma(0)} \| \L_{\phi_\gamma}(\u) \|_1 +
     \frac{\gamma}{\phi_\gamma(0)} \sqrt{\sum_{t \in I} \|\x_t\|^2} \,.
\end{equation}
If we further assume that $\| \x_t \| \leq r$ for all $t \in [1, T]$,
for some $r > 0$, this implies
\begin{equation}
M_T \leq \inf_{\gamma > 0, \|\u\| \leq 1} \bigg( \frac{\gamma r}{\phi_\gamma(0)} +
\sqrt{\frac{\|\L_{\phi_\gamma} (\u)\|_1}{\phi_\gamma(0)}} \bigg)^2 \,.
\end{equation}
\end{theorem}
\begin{proof}
For all $\gamma > 0$ and $\u$ with $\| \u \| \leq 1$, the following
statements hold.
By convexity of $\phi_\gamma$ we have $\frac{1}{M_T} \sum_{t \in I}
\phi_\gamma(y_t
\u \cdot \x_t) \geq \phi_\gamma( \u \cdot \z)$, where $\z = \frac{1}{M_T}
\sum_{t \in I} y_t \x_t$. Then, by using the Lipschitz property of
$\phi_\gamma$ we have,
\begin{align*}
  \phi_\gamma(\u \cdot \z) 
  & = \phi_\gamma(\u \cdot \z) - \phi_\gamma(0) + \phi_\gamma(0) \\
  & = - \big| \phi_\gamma(0) - \phi_\gamma(\u \cdot \z) \big| +
  \phi_\gamma(0) \\
  & \geq - \gamma | \u \cdot  \z | + \phi_\gamma(0) \,.
\end{align*}
Combining the two inequalities above and multiplying both sides by
$M_T$ implies
\begin{equation*}
M \phi_\gamma(0) \leq \sum_{t \in I} \phi_\gamma(y_t \u \cdot \x_t)
  + \gamma \Big| \sum_{t \in I} y_t \u \cdot \x_t \Big| \,.
\end{equation*}
Finally, using the Cauchy-Schwartz inequality and Lemma~\ref{lemma:1}
yields
\begin{equation*}
\Big| \sum_{t \in I} y_t \u \cdot \x_t \Big|
 = \Big| \u \cdot \Big(\sum_{t \in I} y_t \x_t \Big) \Big|
 \leq \| \u \| \Big\| \sum_{t \in I} y_t \x_t \Big\| \leq
\sqrt{\sum_{t \in I} \| \x_t \|^2} \,,
\end{equation*}
which completes the proof of the first statement after re-arranging
terms.

If it is further assumed that $\|\x_t\| \leq r$ for all $t \in I$,
then this implies $M \phi_\gamma(0) - r \sqrt{M} - \sum_{t \in I}
\phi_\gamma(y_t \u \cdot \x_t) \leq 0$.  Solving this quadratic
expression in terms of $\sqrt{M}$ proves the second statement.  \qed
\end{proof}

It is straightforward to see that the $\rho$-margin hinge loss
$\phi_\rho(x) = (1 - x / \rho)_+$ is $(1/\rho)$-admissible with
$\phi_\rho(0) = 1$ for all $\rho$, which gives the following
corollary.

\begin{corollary}
\label{cor:hinge}
Let $I$ denote the set of rounds at which the Perceptron algorithm
makes an update when processing a sequence of training instances
$\x_1, \ldots, \x_T \in \Rset^N$.  For any $\rho > 0$ and any $\u \in
\Rset^{N}$ with $\| \u \| \leq 1$, consider the vector of $\rho$-hinge
losses incurred by $\u$: $\L_{\rho}(\u) = \big[ (1 - \frac{y_t (\u
  \cdot \x_t)}{\rho})_+ \big]_{t \in I}$.  Then, the number of updates
$M_T = |I|$ made by the Perceptron algorithm can be bounded as
follows:
\begin{equation}
\label{eq:l10}
M_T \leq \inf_{\rho > 0 \|\u\| \leq 1}
  \| \L_{\rho}(\u) \|_1 +
     \frac{\sqrt{\sum_{t \in I} \|\x_t\|^2}}{\rho} \,.
\end{equation}
If we further assume that $\| \x_t \| \leq r$ for all $t \in [1, T]$,
for some $r > 0$, this implies
\begin{equation}
\label{eq:l1}
M_T \leq \inf_{\rho > 0, \|\u\| \leq 1} \bigg( \frac{r}{\rho} +
\sqrt{\|\L_{\rho} (\u)\|_1} \bigg)^2 \,.
\end{equation}
\end{corollary}
The mistake bound \eqref{eq:l10} appears already in
\cite{CesaBianchiConconiGentile2004} but we could not find its proof
either in that paper or in those it references for this bound.

Another application of Theorem~\ref{th:l1} is to the squared-hinge
loss $\phi_\rho(x) = (1 - x / \rho)_+^2$.  Assume that $\|\x\| \leq
r$, then the inequality $\|y (\u \cdot \x) \| \leq \|\u\| \|\x\| \leq
r$ implies that the derivative of the hinge-loss is also bounded,
achieving a maximum absolute value $|\phi'_\rho(r)| = |\frac{2}{\rho}
(\frac{r}{\rho} - 1)| \leq \frac{2r}{\rho^2}$.  Thus, the
$\rho$-margin squared hinge loss is $(2r / \rho^2)$-admissible with
$\phi_\rho(0) = 1$ for all $\rho$. This leads to the following
corollary.

\begin{corollary}
  Let $I$ denote the set of rounds at which the Perceptron algorithm
  makes an update when processing a sequence of training instances
  $\x_1, \ldots, \x_T \in \Rset^N$ with $\|\x_t\| \leq r$ for all $t
  \in [1, T]$.  For any $\rho > 0$ and any $\u \in \Rset^{N}$ with $\|
  \u \| \leq 1$, consider the vector of $\rho$-margin squared hinge
  losses incurred by $\u$: $\L_{\rho}(\u) = \big[ (1 - \frac{y_t (\u
    \cdot \x_t)}{\rho})^2_+ \big]_{t \in I}$.  Then, the number of
  updates $M_T = |I|$ made by the Perceptron algorithm can be bounded
  as follows:
\begin{equation}
M_T \leq \inf_{\rho > 0 \|\u\| \leq 1}
  \| \L_{\rho}(\u) \|_1 +
     \frac{2r \sqrt{\sum_{t \in I} \|\x_t\|^2}}{\rho^2} \,.
\end{equation}
This also implies
\begin{equation}
M_T \leq \inf_{\rho > 0, \|\u\| \leq 1} \bigg( \frac{2r^2}{\rho^2} +
\sqrt{\|\L_{\rho} (\u)\|_1} \bigg)^2 \,.
\end{equation}
\end{corollary}
Theorem~\ref{th:l1} can be similarly used to derive mistake bounds in
terms of other admissible losses.

\subsection{$L_2$-norm mistake bounds}

The original results of this section are due to
\citet{FreundSchapire99}.  Here, we extend their proof to derive finer
mistake bounds for the Perceptron algorithm in terms of the $L_2$-norm
of the vector of hinge losses of an arbitrary weight vector at points
where an update is made.

\begin{theorem}
\label{th:l2}
Let $I$ denote the set of rounds at which the Perceptron algorithm
makes an update when processing a sequence of training instances
$\x_1, \ldots, \x_T \in \Rset^N$.  For any $\rho > 0$ and any $\u \in
\Rset^{N}$ with $\| \u \| \leq 1$, consider the vector of $\rho$-hinge
losses incurred by $\u$: $\L_{\rho}(\u) = \big[ (1 - \frac{y_t (\u
  \cdot \x_t)}{\rho})_+ \big]_{t \in I}$.  Then, the number of updates
$M_T = |I|$ made by the Perceptron algorithm can be bounded as
follows:
\begin{equation}
M_T \leq \inf_{\rho > 0, \|\u\| \leq 1}
  \left( \frac{\| \L_\rho(\u) \|_2}{2} + 
  \sqrt{\frac{\| \L_\rho(\u) \|_2^2}{4} +
     \frac{\sqrt{\sum_{t \in I} \|\x_t\|^2}}{\rho}}
  \right)^2 \,.
\end{equation}
If we further assume that $\| \x_t \| \leq r$ for all $t \in [1, T]$,
for some $r > 0$, this implies
\begin{equation}
\label{eq:l2}
M_T \leq \inf_{\rho > 0, \|\u\| \leq 1} \bigg( \frac{r}{\rho} +
\|\L_\rho (\u)\|_2 \bigg)^2 \,.
\end{equation}
\end{theorem}
\begin{proof}
We first reduce the problem to the separable case by mapping 
each input vector $\x_t \in \Rset^N$ to a vector in $\x'_t \in \Rset^{N + T}$
as follows:
\begin{equation*}
\x_t = 
\begin{bmatrix}
x_{t, 1}\\
\vdots\\
x_{t, N}
\end{bmatrix}
\mapsto 
\x'_t = 
\begin{bmatrix}
x_{t, 1} &
\ldots &
x_{t, N} &
0 &
\ldots &
0 &
\underbrace{\Delta}_{\substack{\text{$(N + t)$th}\\ \text{component}}} &
0 &
\ldots &
0 
\end{bmatrix}^\top\,,
\end{equation*}
where the first $N$ components of $\x'_t$ coincide with those of $\x$
and the only other non-zero component is the $(N + t)$th component
which is set to $\Delta$, a parameter $\Delta$ whose value will be
determined later. Define $l_t$ by $l_t = (1 - \frac{y_t \u \cdot
  \x_t}{\rho}) \1_{t \in I}$. Then, the vector $\u$ is replaced by the
vector $\u'$ defined by
\begin{equation*}
\u' = 
\begin{bmatrix}
\frac{u_{1}}{Z} ~&
\ldots ~&
\frac{u_{N}}{Z} ~&
\frac{y_1 l_1 \rho}{\Delta Z} ~&
\ldots ~&
\frac{y_T l_T \rho}{\Delta Z} &
\end{bmatrix}^\top.
\end{equation*}
The first $N$ components of $\u'$ are equal to the components of $\u /
Z$ and the remaining $T$ components are functions of the labels and
hinge losses. The normalization factor $Z$ is chosen to guarantee that
$\| \u' \| = 1$: $Z = \sqrt{1 + \frac{\rho^2
    \|\L_\rho(\u)\|^2}{\Delta^2}}$.  Since the additional coordinates
of the instances are non-zero exactly once, the predictions made by
the Perceptron\index{Perceptron algorithm} algorithm for $\x'_t$, $t
\in [1, T]$ coincide with those made in the original space for $\x_t$,
$t \in [1, T]$. In particular, a change made to the additional
coordinates of $\w'$ does no affect any subsequent prediction.
Furthermore, by definition of $\u'$ and $\x'_t$, we can write for any
$t \in I$:
\begin{align*}
y_t (\u' \cdot \x'_t) & = y_t \Big(\frac{\u \cdot \x_t }{Z} + \Delta
\frac{y_t l_t \rho}{Z \Delta} \Big)\\
& = \frac{y_t \u \cdot \x_t }{Z} + \frac{l_t \rho}{Z}\\
& \geq \frac{y_t \u \cdot \x_t }{Z} + \frac{\rho - y_t (\u \cdot \x_t)}{Z} =
\frac{\rho}{Z},
\end{align*} 
where the inequality results from the definition of
$l_t$.  Summing up the inequalities for all $t \in I$ and using
Lemma~\ref{lemma:1} yields $M_T \frac{\rho}{Z} \leq \sum_{t \in I} y_t
(\u' \cdot \x'_t) \leq \sqrt{ \sum_{t \in I} \|\x'_t\|^2}$.
Substituting the value of $Z$ and re-writing in terms of $\x$ implies:
\begin{align*}
  M_T^2 
& \leq \Big(\frac{1}{\rho^2} +
\frac{\|\L_\rho(\u)\|^2}{\Delta^2} \Big) \Big(R^2 + M_T \Delta^2 \Big)
\\
& = \frac{R^2}{\rho^2} + \frac{R^2 \|\L_\rho(\u)\|^2}{\Delta^2} +
\frac{M_T \Delta^2}{\rho^2} + M^T \|\L_\rho(\u)\|^2 \,,
\end{align*}
where $R = \sqrt{\sum_{t \in I} \|\x_t\|^2}$. Now, solving for
$\Delta$ to minimize this bound gives $\Delta^2 = \frac{\rho
\|\L_\rho(\u)\| R }{\sqrt{M_T}}$ and further simplifies the bound
\begin{align*}
M_T^2 
 & \leq \frac{R^2}{\rho^2} + 2 \frac{\sqrt{M_T} \|\L_\rho(\u)\|
R}{\rho} + M_T \|\L_\rho(\u)\|^2 \\
 & = \Big( \frac{R}{\rho} + \sqrt{M_T} \|\L_\rho(\u)\|_2 \Big)^2 \,.
\end{align*}
Solving the second-degree inequality $M_T - \sqrt{M_T}
\|\L_\rho(\u)\|_2 - \frac{R}{\rho} \leq 0$ proves the first statement
of the theorem.  The second theorem is obtained by first bounding $R$
with $r \sqrt{M_T}$ and then solving the second-degree inequality.
\qed
\end{proof} 

\subsection{Discussion}

One natural question this survey raises is the respective quality of
the $L_1$- and $L_2$-norm bounds. The comparison of \eqref{eq:l1} and
\eqref{eq:l2} for the $\rho$-margin hinge loss shows that, for a fixed
$\rho$, the bounds differ only by the following two quantities:
\begin{align*}
 \min_{\|\u\| \leq 1} \|\L_\rho(\u)\|_1
  & = \min_{\|\u\| \leq 1} \sum_{t \in I} (1 - y_t \big(\u \cdot \x_t) / \rho \big)_+\\
  \min_{\|\u\| \leq 1} \|\L_\rho(\u)\|_2^2 
  & = \min_{\|\u\| \leq 1} \sum_{t \in I} (1 - y_t \big(\u \cdot \x_t) / \rho \big)_+^2.
\end{align*}
These two quantities are data-dependent and in general not
comparable. For a vector $\u$ for which the individual losses $(1 -
y_t \big(\u \cdot \x_t))$ are all less than one, we have
$\|\L_\rho(\u)\|_2^2 \leq \|\L_\rho(\u)\|_1$, while the contrary holds
if the individual losses are larger than one.

\section{Generalization Bounds}

In this section, we consider the case where the training sample
processed is drawn according to some distribution $D$.  Under some
mild conditions on the loss function, the hypotheses returned by an
on-line learning algorithm can then be combined to define a hypothesis
whose generalization error can be bounded in terms of its regret. Such
a hypothesis can be determined via cross-validation
\cite{Littlestone89} or using the online-to-batch theorem of
\cite{CesaBianchiConconiGentile2004}. The latter can be
combined with any of the mistake bounds presented in the previous
section to derive generalization bounds for the Perceptron predictor.

Given $\delta > 0$, a sequence of labeled examples $(x_1, y_1),
\ldots, (y_T, x_T)$, a sequence of hypotheses $h_1, \ldots, h_T$, and
a loss function $L$, define the \emph{penalized risk minimizing
  hypothesis} as $\h h = h_{i^*}$ with
\begin{equation*}
i^* = \argmin_{i \in [1, T]} \frac{1}{T - i + 1}\sum_{t = i}^{T} L(y_t
h_i(x_t)) + \sqrt{\frac{\log \frac{T(T + 1)}{\delta}}{2(T - i + 1)} }.
\end{equation*}
The following theorem gives a bound on the expected loss of $\h h$ on
future examples.

\begin{theorem}[\cite{CesaBianchiConconiGentile2004}]
\label{th:online2batch_convex}
Let $S$ be a labeled sample $((x_1, y_1), \ldots, (x_T, y_T))$ drawn
i.i.d.\ according to $D$, $L$ a loss function bounded by one, and
$h_1, \ldots, h_T$ the sequence of hypotheses generated by an on-line
algorithm $\Alg$ sequentially processing $S$.  Then, for any $\delta >
0$, with probability at least $1 - \delta$, the following holds:
\begin{equation}
  \E_{(x, y) \sim D} [ L(y \h h(x)) ] \leq \frac{1}{T}\sum_{i=1}^T L(y_i h_i(x_i)) 
  + 6 \sqrt{\frac{1}{T} \log \frac{2 (T + 1)}{\delta}} \,.
\label{eq:online2batch_nonconvex}
\end{equation}
\end{theorem}

Note that this theorem does not require the loss function to be
convex.  Thus, if $L$ is the zero-one loss, then the empirical loss
term is precisely the average number of mistakes made by the
algorithm.  Plugging in any of the mistake bounds from the previous
sections then gives us a learning guarantee with respect to the
performance of the best hypothesis as measured by a margin-loss (or
any $\gamma$-admissible loss if using Theorem~\ref{th:l1}).  Let $\h
\w$ denote the weight vector corresponding to the penalized risk
minimizing Perceptron hypothesis chosen from all the intermediate
hypotheses generated by the algorithm. Then, in view of
Theorem~\ref{th:l1}, the following corollary holds.

\begin{corollary}
\label{cor:l1generalization}
Let $I$ denote the set of rounds at which the Perceptron algorithm
makes an update when processing a sequence of training instances
$\x_1, \ldots, \x_T \in \Rset^N$. For any vector $\u \in \Rset^{N}$
with $\| \u \| \leq 1$ and any $\gamma$-admissible loss function
$\phi_\gamma$, consider the vector of losses incurred by $\u$: $
\L_{\phi_\gamma}(\u) = \big[ \phi_\gamma(y_t (\u \cdot \x_t)) \big]_{t
  \in I}$. Then, for any $\delta > 0$, with probability at least $1 -
\delta$, the following generalization bound holds for the penalized
risk minimizing Perceptron hypothesis $\h \w$:
\begin{multline*}
\Pr_{(x, y) \sim D}[y (\h \w \cdot \x) < 0]  \\
  \leq 
\inf_{\gamma > 0, \|\u\| \leq 1}
\frac{\| \L_{\phi_\gamma}(\u) \|_1}{\phi_\gamma(0) T} +
     \frac{\gamma \sqrt{\sum_{t \in I} \|\x_t\|^2}}{\phi_\gamma(0) T} 
  + 6 \sqrt{\frac{1}{T} \log \frac{2 (T + 1)}{\delta}} \,.
\end{multline*}
\end{corollary}
Any $\gamma$-admissible loss can be used to derive a more explicit
form of this bound in special cases, in particular the hinge loss
or the squared hinge loss.
Using Theorem~\ref{th:l2}, we obtain the following $L_2$-norm
generalization bound.

\begin{figure}[t]
\begin{ALGO}{KernelPerceptron\index{Perceptron algorithm!kernel}}{\Alpha_0}
\SET{\Alpha}{\Alpha_0 \qquad \triangleright \text{typically } \Alpha_ 0 = \mat{0}}
\DOFOR{t \EQ 1 \TO T}
\CALL{Receive}{x_t}
\SET{\h y_t}{\sgn(\sum_{s = 1}^T \alpha_s y_s K(x_s, x_t) )}
\CALL{Receive}{y_t}
\IF{(\h y_t \neq y_t)}
	\SET{\alpha_{t}}{\alpha_t + 1}
\FI
\OD
\RETURN{\Alpha}
\end{ALGO}
\caption{Kernel Perceptron\index{Perceptron algorithm!kernel} algorithm for PDS kernel $K$.}
\label{fig:kernel_perceptron}
\end{figure}

\begin{corollary}
\label{cor:l2generalization}
Let $I$ denote the set of rounds at which the Perceptron algorithm
makes an update when processing a sequence of training instances
$\x_1, \ldots, \x_T \in \Rset^N$.  For any $\rho > 0$ and any $\u \in
\Rset^{N}$ with $\| \u \| \leq 1$, consider the vector of $\rho$-hinge
losses incurred by $\u$: $\L_{\rho}(\u) = \big[ (1 - \frac{y_t (\u
  \cdot \x_t)}{\rho})_+ \big]_{t \in I}$. Then, for any $\delta > 0$,
with probability at least $1 - \delta$, the following generalization
bound holds for the penalized risk minimizing Perceptron hypothesis
$\h \w$:
\begin{multline*}
\Pr_{(x, y) \sim D}[y (\h \w \cdot \x) < 0]  \\
  \leq 
\inf_{\rho > 0, \|\u\| \leq 1}
  \frac{1}{T} \left( \frac{\| \L_\rho(\u) \|_2}{2} + 
  \sqrt{\frac{\| \L_\rho(\u) \|_2^2}{4} +
     \frac{\sqrt{\sum_{t \in I} \|\x_t\|^2}}{\rho}}
  \right)^2\\
 + 6 \sqrt{\frac{1}{T} \log \frac{2 (T + 1)}{\delta}} \,.
\end{multline*}
\end{corollary}

\section{Kernel Perceptron algorithm}

The Perceptron algorithm of Figure~\ref{fig:perceptron_algorithm} can
be straightforwardly extended to define a non-linear separator using a
positive definite kernel $K$ \citep{AizermanBravermanRozonoer64}.
Figure~\ref{fig:kernel_perceptron} gives the pseudocode of that
algorithm known as the \emph{kernel Perceptron algorithm}. The
classifier $\sgn(h)$ learned by the algorithm is defined by $h \colon
x \mapsto \sum_{t = 1}^T \alpha_t y_t K(x_t, x)$. The results of the
previous sections apply similarly to the kernel perceptron algorithm
with $\| \x_t \|^2$ replaced with $K(x_t, x_t)$. In particular,
the quantity $\sqrt{\sum_{t \in I} \|\x_t\|^2}$ appearing in several
of the learning guarantees can be replaced with the familiar trace
$\Tr[\K]$ of the kernel matrix $\K = [K(x_i, x_j)]_{i, j \in I}$
over the set of points at which an update is made, which is a
standard term appearing in margin bounds for kernel-based hypothesis
sets.

\bibliographystyle{plainnat}
\bibliography{pmb}
\end{document}